\documentclass[10pt,twocolumn,letterpaper]{article}

\usepackage[pagenumbers]{cvpr}      

\usepackage{graphicx}
\usepackage{amsmath}
\usepackage{amssymb}
\usepackage{booktabs}
\usepackage{appendix}
\usepackage{amsthm}
\usepackage{bbm}
\usepackage[table]{xcolor}
\usepackage[accsupp]{axessibility}  

\newtheorem{theorem}{Theorem}
\newtheorem{Theorem}{Theorem}
\newtheorem{Corollary}{Corollary}

\newtheorem{definition}{Definition}
\newtheorem{corollary}{Corollary}

\DeclareMathOperator*{\argmin}{arg\!\min}

\usepackage[pagebackref,breaklinks,colorlinks]{hyperref}
\usepackage{amsfonts,amssymb}

\usepackage[capitalize]{cleveref}
\crefname{section}{Sec.}{Secs.}
\Crefname{section}{Section}{Sections}
\Crefname{table}{Table}{Tables}
\crefname{table}{Tab.}{Tabs.}

\def\cvprPaperID{7568} 
\def\confName{CVPR}
\def\confYear{2023}

\begin{document}

\title{Solving Oscillation Problem in Post-Training Quantization Through a Theoretical Perspective}

\author{
    Yuexiao Ma\textsuperscript{\rm 1}\thanks{This work was done when Yuexiao Ma was intern at ByteDance Inc. Code is available at: \url{https://github.com/bytedance/MRECG}},
    Huixia Li\textsuperscript{\rm 2},
    Xiawu Zheng\textsuperscript{\rm 3},
    Xuefeng Xiao\textsuperscript{\rm 2},
    Rui Wang\textsuperscript{\rm 2}, \\
    Shilei Wen\textsuperscript{\rm 2},
    Xin Pan\textsuperscript{\rm 2},
    Fei Chao\textsuperscript{\rm 1},
    Rongrong Ji\textsuperscript{\rm 1,4}\thanks{Corresponding Author: rrji@xmu.edu.cn} \\
    \textsuperscript{\rm 1} Key Laboratory of Multimedia Trusted Perception and Efficient Computing, \\Ministry of Education of China, School of Informatics, Xiamen University, 361005, P.R. China.\\
    \textsuperscript{\rm 2} ByteDance Inc.
    \textsuperscript{\rm 3} Peng Cheng Laboratory, Shenzhen, China. \\
    \textsuperscript{\rm 4} Shenzhen Research Institute of Xiamen University.\\
    {\tt\small bobma@stu.xmu.edu.cn, zhengxw01@pcl.ac.cn, \{fchao, rrji\}@xmu.edu.cn,} \\
    {\tt\small\{lihuixia, xiaoxuefeng.ailab, ruiwang.rw, zhengmin.666, panxin.321\}@bytedance.com.} \\
}
\maketitle

\begin{abstract}
Post-training quantization (PTQ) is widely regarded as one of the most efficient compression methods practically, benefitting from its data privacy and low computation costs. We argue that an overlooked problem of oscillation is in the PTQ methods. In this paper, we take the initiative to explore and present a theoretical proof to explain why such a problem is essential in PTQ. And then, we try to solve this problem by introducing a principled and generalized framework theoretically. In particular, we first formulate the oscillation in PTQ and prove the problem is caused by the difference in module capacity. To this end, we define the module capacity (ModCap) under data-dependent and data-free scenarios, where the differentials between adjacent modules are used to measure the degree of oscillation. The problem is then solved by selecting top-k differentials, in which the corresponding modules are jointly optimized and quantized. Extensive experiments demonstrate that our method successfully reduces the performance drop and is generalized to different neural networks and PTQ methods. For example, with $2$/$4$ bit ResNet-$50$ quantization, our method surpasses the previous state-of-the-art method by $1.9\%$. It becomes more significant on small model quantization, e.g. surpasses BRECQ method by $6.61\%$ on MobileNetV$2$ $\times 0.5$.
\end{abstract}

\begin{figure}[t]
\centering
\includegraphics[width=1\linewidth]{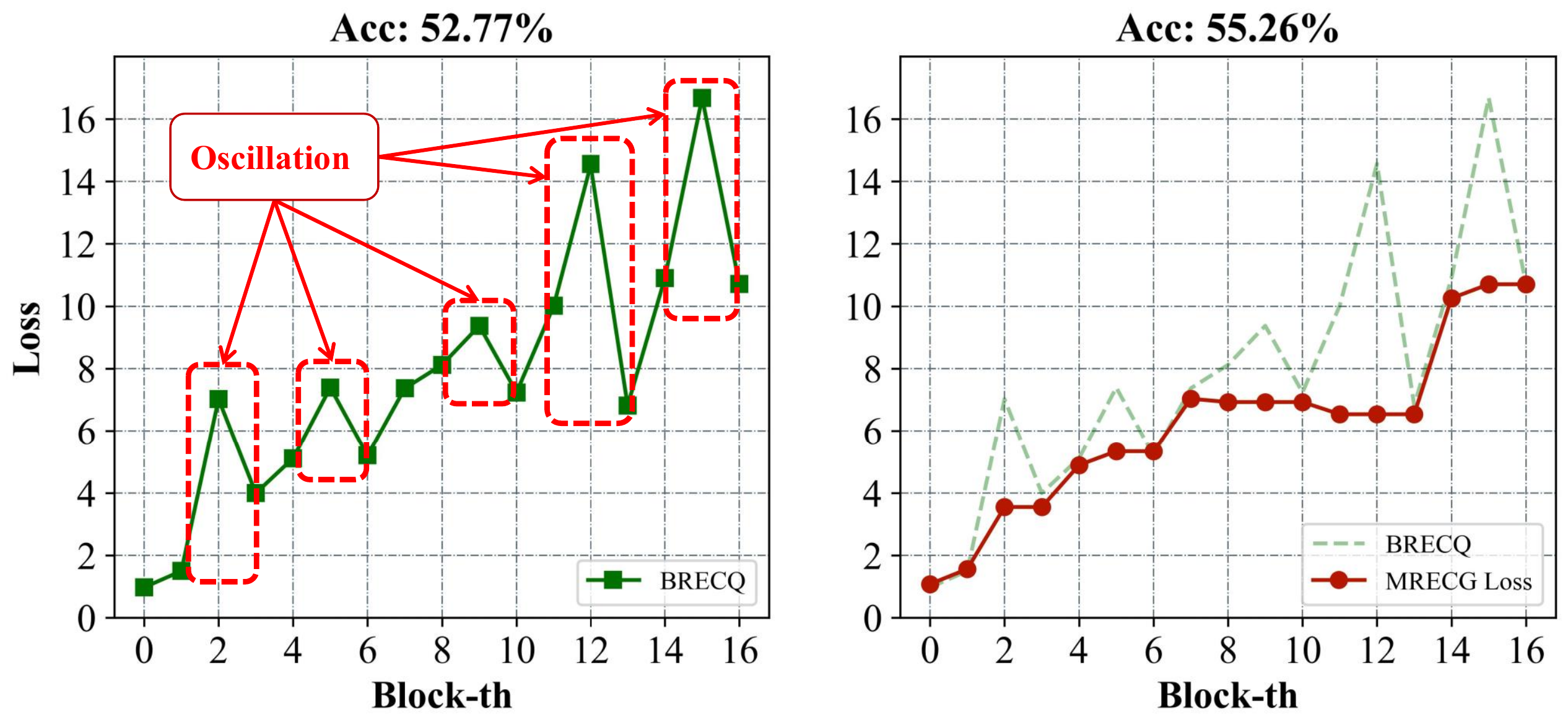}
\caption{\textbf{Left:} Reconstruction loss distribution of BRECQ \cite{li2021brecq} on $0.5$ scaled MobileNetV$2$ quantized to $4$/$4$ bit. Loss oscillation in BRECQ during reconstruction see red dashed box. \textbf{Right:} Mixed reconstruction granularity (MRECG) smoothing loss oscillation and achieving higher accuracy.}
\label{fig:loss_distribution}
\end{figure}

\section{Introduction}
\label{sec:intro}

Deep Neural Networks (DNNs) have rapidly become a research hotspot in recent years, being applied to various scenarios in practice. 
However, as DNNs evolve, better model performance is usually associated with huge resource consumption from deeper and wider networks \cite{he2016deep,krizhevsky2017imagenet,szegedy2017inception}. 
Meanwhile, the research field of neural network compression and acceleration, which aims to deploy models in resource-constrained scenarios, is gradually gaining more and more attention, including but not limited to Neural Architecture Search \cite{liu2018darts, liu2018progressive, zheng2019multinomial, Zheng_2022_CVPR, zheng2020rethinking, zheng2021migo, Evolvingzheng, xia2022progressive, zhang2021you, zheng2023ddpnas, zhang2023targeted}, network pruning \cite{han2015deep, zheng2021information, Elkerdawy_2022_CVPR, Li_2022_CVPR, Wimmer_2022_CVPR, Shen_2022_CVPR}, and quantization \cite{esser2019learned, gong2019differentiable, choi2018pact, wang2019haq, nagel2020up, li2021brecq, wei2022qdrop, li2020pams, ma2021ompq}. 
Among these methods, quantization proposed to transform float network activations and weights to low-bit fixed points, which is capable of accelerating inference \cite{krishnamoorthi2018quantizing} or training \cite{zhu2020towards} speed with little performance degradation.

In general, network quantization methods are divided into quantization-aware training (QAT) \cite{choi2018pact, gong2019differentiable, esser2019learned} and post-training quantization (PTQ) \cite{li2021brecq,nagel2020up, wei2022qdrop, hubara2021accurate}. 
The former reduces the quantization error by quantization fine-tuning. Despite the remarkable results, the massive data requirements and high computational costs hinder the pervasive deployment of DNNs, especially on resource-constrained devices.
Therefore, PTQ is proposed to solve the aforementioned problem, which requires only minor or zero calibration data for model reconstruction. 
Since there is no iterative process of quantization training, PTQ algorithms are extremely efficient, usually obtaining and deploying quantized models in a few minutes. 
However, this efficiency often comes at the partial sacrifice of accuracy. PTQ typically performs worse than full precision models without quantization training, especially in low-bit compact model quantization. 
Some recent algorithms \cite{nagel2020up,li2021brecq,wei2022qdrop} try to address this problem. 
For example, Nagel \emph{et al.}~\cite{nagel2020up} constructs new optimization functions by second-order Taylor expansions of the loss functions before and after quantization, which introduces soft quantization with learnable parameters to achieve adaptive weight rounding. 
Li \emph{et al.}~\cite{li2021brecq} changes layer-by-layer to block-by-block reconstruction and uses diagonal Fisher matrices to approximate the Hessian matrix to retain more information. 
Wei \emph{et al.}~\cite{wei2022qdrop} discovers that randomly disabling some elements of the activation quantization can smooth the loss surface of the quantization weights. 

However, we observe that all the above methods show different degrees of oscillation with the deepening of the layer or block during the reconstruction process, as illustrated in the left sub-figure of Fig.~\ref{fig:loss_distribution}. 
We argue that the problem is essential and has been overlooked in the previous PTQ methods. 
In this paper, through strict mathematical definitions and proofs, we answer $3$ questions about the oscillation problem, which are listed as follows:

(\romannumeral1). \emph{\textbf{Why the oscillation happens in PTQ?}} To answer this question, we first define module topological homogeneity, which relaxes the module equivalence restriction to a certain extent. And then, we give the definition of module capacity under the condition of module topological homogeneity. In this case, we can prove that when the capacity of the later module is large enough, the reconstruction loss will break through the effect of quantization error accumulation and decrease. On the contrary, if the capacity of the later module is smaller than that of the preceding module, the reconstruction loss increases sharply due to the amplified quantization error accumulation effect. Overall, we demonstrate that the oscillation of the loss during PTQ reconstruction is \textbf{caused by the difference in module capacity;}

(\romannumeral2). \emph{\textbf{How the oscillation will influence the final performance?}}
We observe that the final reconstruction error is highly correlated with the largest reconstruction error in all the previous modules by randomly sampling a large number of mixed reconstruction granularity schemes. In other words, when oscillation occurs, the previous modules obviously have larger reconstruction errors, thus leading to worse accuracy in PTQ;

(\romannumeral3). \emph{\textbf{How to solve the oscillation problem in PTQ?}} Since oscillation is caused by the different capacities of the front and rear modules, we propose the \textbf{M}ixed \textbf{REC}onstruction \textbf{G}ranularity (MRECG) method which jointly optimizes the modules where oscillation occurs. Besides, our method is applicable in data-free and data-dependent scenarios, which is also compatible with different PTQ methods. 
In general, our contributions are listed as follows:
\begin{itemize}
\item We reveal for the first time the oscillation problem in PTQ, which has been neglected in previous algorithms. However, we discover that smoothing out this oscillation is essential in the optimization of PTQ.
\item We show theoretically that this oscillation is caused by the difference in the capability of adjacent modules. A small module capability exacerbates the cumulative effect of quantization errors making the loss increase rapidly, while a large module capability reduces the cumulative quantization errors making the loss decrease.
\item To solve the oscillation problem, we propose a novel \textbf{M}ixed \textbf{REC}onstruction \textbf{G}ranularity (MRECG) method, which employs loss metric and module capacity to optimize mixed reconstruction granularity under data-dependency and data-free scenarios. The former finds the global optimum  with moderately higher overhead and thus has the best performance. The latter is more effective with a minor performance drop.
\item We validate the effectiveness of the proposed method on a wide range of compression tasks in ImageNet. In particular, we achieve a Top-$1$ accuracy of 58.49\% in MobileNetV$2$ with $2$/$4$ bit, which exceeds current SOTA methods by a large margin. Besides, we also confirm that our algorithm indeed eliminates the oscillation of reconstruction loss on different models and makes the reconstruction process more stable.
\end{itemize}

\begin{figure*}[t]
\centering
\includegraphics[width=1\linewidth]{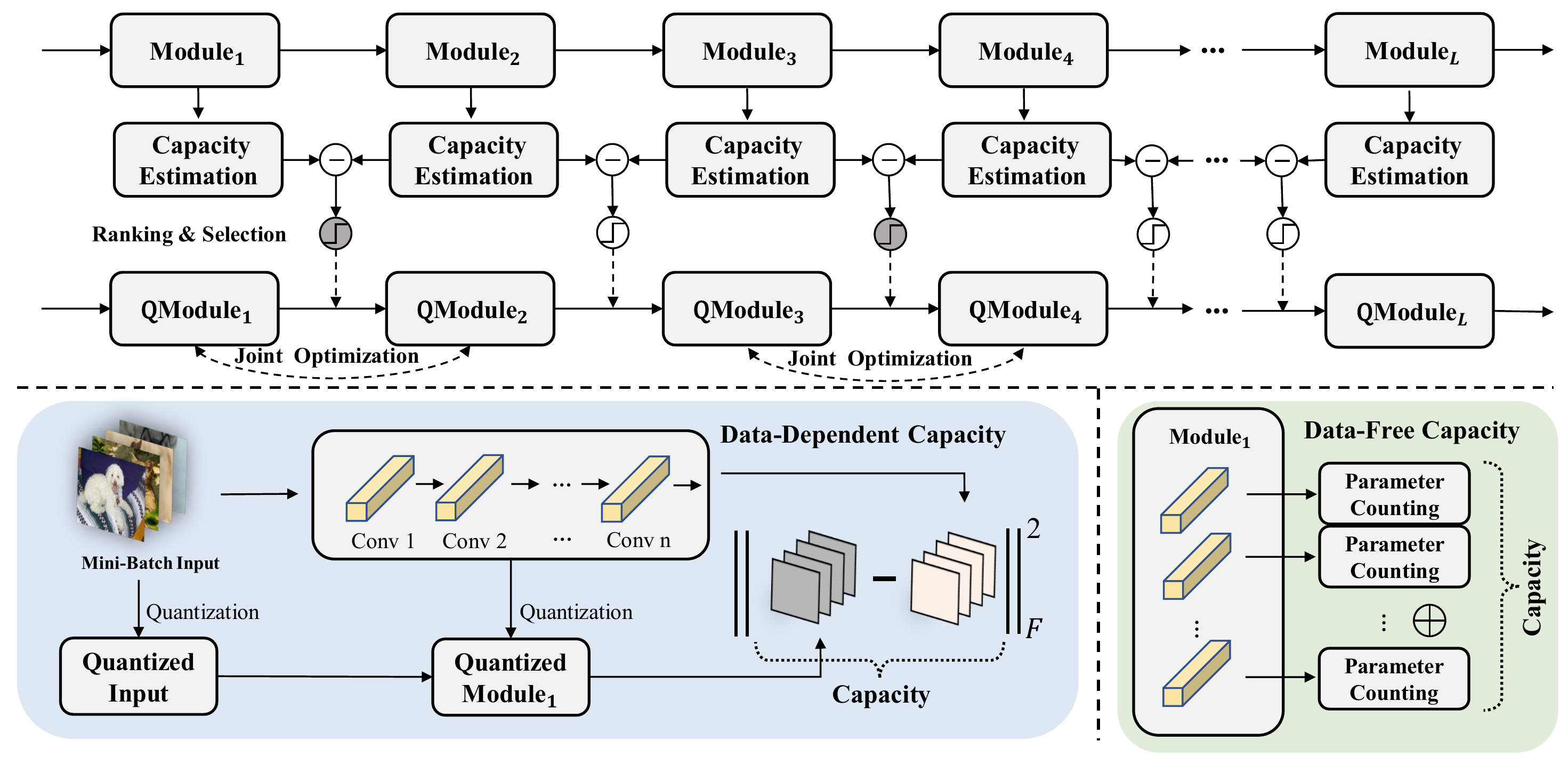}
\caption{Overview of our method. Top: Optimization process for MRECG. We first estimate the capacity of each module, then we rank the capacity difference of adjacent modules, and finally we jointly optimize the adjacent modules corresponding to the top-k capacity difference. Bottom left: Module capacity estimation for data-dependent scenarios. We calculate the squared Frobenius norm of the feature map difference before and after quantization in the PTQ process. Bottom right: Module capacity estimation for the data-free scenario. We sum the convolutional layer capacities corresponding to the modules.}
\label{fig:overview}
\end{figure*}

\section{Related Work}
\label{sec:related_work}

\noindent
\textbf{Quantization:} Quantization can be divided into two categories: \textbf{Q}uantization-\textbf{A}ware \textbf{T}raining (QAT) and \textbf{P}ost-\textbf{T}raining \textbf{Q}uantization (PTQ). QAT~\cite{choi2018pact, esser2019learned, zhuang2018towards} uses the entire training dataset for quantization training and updates the gradients by back-propagation of the network to eliminate quantization errors.
Although QAT integrates various training methods to achieve higher accuracy, this process is often resource intensive and limited in certain data privacy scenarios. Our research interest is not in this field.

PTQ has attracted increasing attention in recent years due to the advantages of efficient model deployment and low data dependence. Since quantization training is not included, PTQ algorithms usually use a small calibration dataset for reconstruction to obtain a better-quantized model by optimizing the approximation of the second-order Taylor expansion term for task loss, which is introduced in Sec.~\ref{sec:intro}.

\noindent
\textbf{Module Capacity:} Some common parameters affect the module capacity, for example, the size of the filter, the bit width of the weight parameter, and the number of convolutional groups. In addition, some investigations show that the stride and residual shortcuts also affect the module capacity. Kong \emph{et al.}~\cite{kong2017take} show that convolution with stride=2 can be equivalently replaced by convolution with stride=1. Meanwhile, the filter size of the replaced convolution is larger than that of the original convolution, implying an increase in module capacity.
The MobileNetV$2$ proposed by \cite{sandler2018mobilenetv2} contains depth-wise convolution, which does not contain the exchange of information between channels, so it somehow compromises the model performance. 
Liu \emph{et al.}~\cite{liu2018bi} argue that the input of full precision residual shortcuts increases the representation capacity of the quantized module.

\section{Methodology}
\label{sec:methodology}

In this section, we first prove that the oscillation problem of PTQ is highly correlated with the module capacity by a theorem and corollary. 
Secondly, we construct the capacity difference optimization problem and present two solutions in data-dependent and data-free scenarios, respectively. 
Finally, we analyze expanding the batch size of the calibration data to reduce the expectation approximation error, which shows a trend of diminishing marginal utility.

\subsection{Oscillation problem of PTQ}
\label{subsec:osci}

Without loss of generality, we employ modules as the basic unit of our analysis. 
In particular, different from BRECQ, module granularity in this paper is more flexible, which represents the layer, block, or even stage granularity. 
Formally, we set $ f_i^{(n_i)}(\cdot) $ be the $i$-th (i=1,2,...,L) module of the neural network which contains $n_i$ convolutional layers. We propose a more general reconstruction loss under the module granularity framework as follows,

\begin{equation}\label{eq:definition}
	 \mathcal{L}(W_i, X_i) = \mathbb{E}\left [ {\Vert f_i^{(n_i)}(W_i, X_i)- f_i^{(n_i)}(\widetilde{W}_i, \widetilde{X}_i)\Vert}_F^2 \right ], 
\end{equation}
where $W_i, X_i$ are the weights and inputs of the $i$-th module, respectively. $ \widetilde{W}_i, \widetilde{X}_i $ are the corresponding quantized version. When $f_i^{(n_i)}(\cdot)$ contains only one convolution layer, Eq.~\ref{eq:definition} degenerates to the optimization function $ {\Vert W_iX_i- \widetilde{W}_i\widetilde{X}_i\Vert}_F^2$ in AdaRound \cite{nagel2020up}. When $f_i^{(n_i)}(\cdot)$ contains all convolutional layers within the $i$-th block, Eq.~\ref{eq:definition} degenerates to the optimization function $\mathbb{E} \left [ \Delta z^{(i),T}H^{z^{(i)}}\Delta z^{(i)} \right ]$ in BRECQ \cite{li2021brecq}. Note that we ignore the regularization term in AdaRound which facilitates convergence since it converges to $0$ after each module is optimized. In addition, we omit the squared gradient scale in BRECQ for simplicity. 
If two modules have the same number of convolutional layers and all associated convolutional layers have the same hyper-parameters, they are called to be equivalent (including but not limited to additional residual input, kernel size, channels, groups, and stride).

\begin{figure}[t]
\centering
\includegraphics[width=1.0\linewidth]{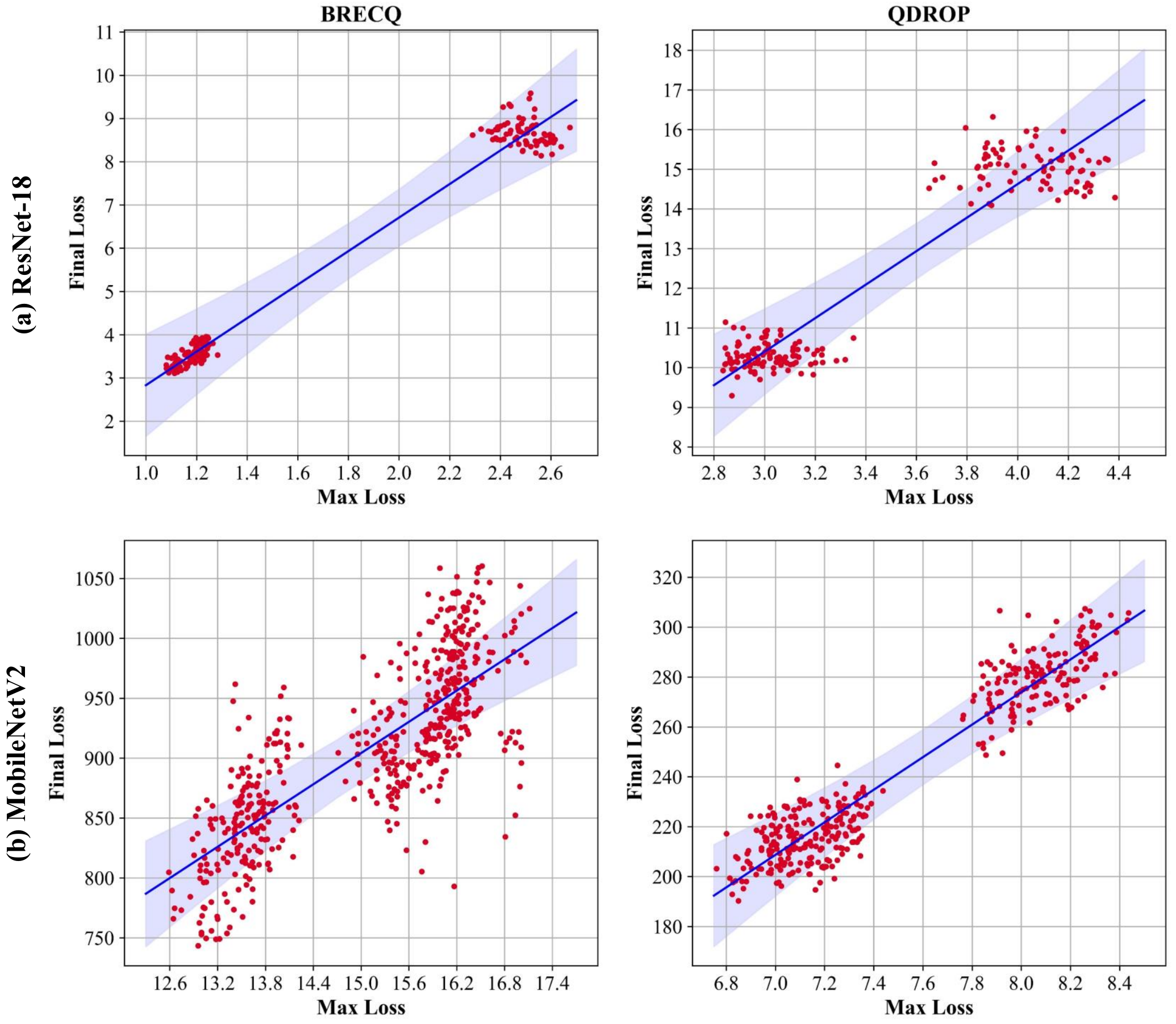}
\caption{The relationship between the final reconstruction error and the maximum reconstruction error of the previous module. We randomly sample a number of mixed reconstruction granularity schemes and recover the accuracy using BRECQ or QDROP. We record the final reconstruction error and the maximum reconstruction error of the previous modules for each scheme. The first row of the figure represents the experiments on $4$/$4$ bit ResNet-$18$. The second row of the figure represents the experiments on $4$/$4$ bit $0.5$ scaled MobileNetV$2$. The first and second columns represent using the BRECQ or QDROP algorithms to recover the accuracy.}
\label{fig:max_final_loss_distribution}
\end{figure}

There exists an accumulation effect of quantization error in quantization \cite{yun2021all}, which is manifested by the increasing tendency of quantization error in the network influenced by the quantization of previous layers. Since PTQ does not contain quantization training, this accumulation effect is more obvious in reconstruction loss. We propose the following theorem to demonstrate that the accumulation effect of quantization error in PTQ leads to an incremental loss.

\begin{theorem}
\label{the:error_cumul}
Given a pre-trained model and input data. If two adjacent modules are equivalent, we have,
\begin{equation}\label{eq:inequality}
	 \mathcal{L}(W_i, X_i) \leq \mathcal{L}(W_{i+1}, X_{i+1}). 
\end{equation}

\end{theorem}
Detailed proof is provided in the {supplementary material}. Theorem~\ref{the:error_cumul} illustrates that under the condition that two adjacent modules are equivalent, the quantization error accumulation leads to the increment of loss

However, due to the difference of one or several convolution hyper-parameters (additional residual input, kernel size, channels, groups, stride, etc.), the condition of adjacent module equivalence in the above theorem is difficult to be satisfied in real scenarios. 
Non-equivalent modules will inevitably result in differences in module capacity. It is well-known that the number of parameters and the bit-width of the module affects the module capacity, in which the effect of these hyper-parameters on the module capacity is easy to quantify.
In addition, we introduce in Sec.~\ref{sec:related_work} that the residual input~\cite{liu2018bi}, the convolution type~\cite{sandler2018mobilenetv2}, and the convolution hyperparameters~\cite{kong2017take} all affect the module capacity, which the effects are difficult to quantify. Therefore, we define the concept of module topological homogeneity as follows, which relaxes the restrictions on module equivalence to some extent, while making it possible to compare capacities between different modules.

\begin{definition} \textbf{\textit{(Module Topological Homogeneity)}} Suppose two modules have the same number of convolutional layers. If the hyper-parameters of the corresponding convolutional layers of the two modules are the same except for the kernel size and channels, we claim that the two modules are topologically homogeneous.
\label{topo_defin}
\end{definition}

From Definition~\ref{topo_defin}, we relax the restriction of module equivalence on the equality of kernel size and channels in module topology homogeneity, which only results in differences in the number of module parameters. 
In other words, we obviate the problem that the impact of hyper-parameters such as residual inputs and groups makes the module capacity hard to be quantified. 
Specifically, if the module weight containing $n$ convolutional layers is $W=[W_1, W_2,\cdots,W_n]$ and $W_i$ is the weight of the $i$-th convolution layer in the module, then the module capacity (ModCap) is defined as the following equation. 

\begin{equation}
\label{eq:ModCap}
\begin{aligned}
\text{ModCap} = \sum_{i=1}^n \text{params}(W_i)\times b_i \times {\alpha}_i, 
\end{aligned}
\end{equation}
where $\text{params}(\cdot)$ is a function that counts the number of parameters, $b_i$ is the bit-width of the $i$-th convolutional layer, and ${\alpha}_i$ is used to make convolutional layers that have different strides comparable. 
Specifically, convolutional layers of different strides are equated by some transformation~\cite{kong2017take}, accompanied by a change in the number of parameters. Therefore, we convert the stride = 2 layer in the network to implicit stride = 1 layer by multiplying the scaling parameter ${\alpha}_i$. Under this conversion, all layers of the network satisfy the definition of topological homogeneity and thus are comparable with each other. Theoretically, according to the analysis of \cite{kong2017take}, ${\alpha}_i$ can be set to $1.6$, which is generalized in different networks in our paper.

Then, we derive a corollary of Theorem~\ref{the:error_cumul} to explain why the oscillations occur in PTQ. 
\begin{corollary}
\label{corollary_lossvar}
Suppose two adjacent modules be topologically homogeneous. If the module capacity of the later module is large enough, the loss will decrease. Conversely, if the latter module capacity is smaller than the former, then the accumulation effect of quantization error is exacerbated.
\end{corollary}
Detailed proofs are provided in the supplementary material. From Corollary~\ref{corollary_lossvar}, we can conclude that the oscillation problem of PTQ on a wide range of models is caused by the excessive difference in capacities with adjacent modules. 
The logical/correlation chain of our paper: Oscillation $\propto$ The largest error $\propto$ Final error $\propto$ Accuracy. \textbf{Oscillation $\propto$ The largest error}: From Fig.~\ref{fig:all_loss_dis} and a similar figure in the supplementary material, the more severe the oscillation of the loss distribution corresponding to the different algorithms, the larger the peak of the loss. That is, the degree of oscillation is positively correlated with the largest error. \textbf{The largest error $\propto$ Final error}: In addition, our observations in Fig.~\ref{fig:max_final_loss_distribution} show that the largest error is positively correlated with the final error on different algorithms for different models. 
\textbf{Final error $\propto$ Accuracy:} \textbf{Theoretically}, by performing a Taylor expansion on the accuracy loss function according to BRECQ~\cite{li2021brecq} and Adaround~\cite{nagel2020up}, we can derive the final reconstruction error that is highly correlated to the performance. \textbf{Empirically}, extensive experiments conducted in the above papers also prove the statement. 
In conclusion, the degree of oscillation of the error is positively correlated with the accuracy. Fig.\ref{fig:loss_distribution} in our paper also demonstrates that reducing the degree of oscillation is beneficial to accuracy. In the next section, we will introduce how to eliminate oscillations in PTQ by optimizing module capacity differences.

\subsection{Mixed REConstruction Granularity}

The oscillation problem analyzed in Sec.~\ref{subsec:osci} indicates that the information loss due to the difference in module capacities will eventually affect the performance of PTQ. 
And since there is no quantization training process in PTQ, this information loss is something that cannot be recovered even by increasing the model capacity of subsequent modules. 
Therefore, we smoothen the loss oscillations by jointly optimizing modules with large capacity differences, thus reducing the final reconstruction loss.

From Theorem~\ref{the:error_cumul} and Corollary~\ref{corollary_lossvar}, it is clear that a small capacity of the later module will aggravate the cumulative effect of the quantization error and thus increase the reconstruction loss sharply. Conversely, a large capacity of the later module will reduce the loss and increase the probability of information loss in the subsequent module. 
Therefore, we hope that the capacities of two adjacent modules are as close as possible.
We construct the capacity difference optimization problem for the model containing $L$ modules based on the capacity metric (CM) as follows,

\begin{equation}
\label{modulecap_opt}
\argmin_{\textbf{m}} \sum_{l=1}^{L-1} {\left ({CM}_l-{CM}_{l+1} \right )}^2m_l+\lambda{ \left(\textbf{m}\cdot\mathbbm{1}-k \right ) }^2, 
\end{equation}
where $\textbf{m}\in\mathbb{R}^{l-1}$ is a binary mask vector. When $m_i$=1, it means we perform joint optimization for the $i$-th and $(i+1)$-th modules. $\mathbbm{1}\in\mathbb{R}^{l-1}$ is a vector with all elements of 1. $k$ is a hyper-parameter that controls the number of jointly optimized modules. $\lambda$ controls the importance ratio of the regularization term and the squared capacity difference optimization objective.

We calculate the square of the difference in capacities of all $L-1$ kinds of adjacent modules for ranking and select the top $k$ adjacent modules for joint optimization. We use ModCap and reconstruction loss as our capacity metrics in the data-free and data-dependent scenarios, respectively. Details are shown in Fig.~\ref{fig:overview}. The mixed reconstruction granularity scheme obtained according to the ModCap metric is computationally efficient because it does not involve the reconstruction. However, once we combine a pair of adjacent modules, this combined module cannot compare ModCap for further combination because it is not topologically homogeneous with the adjacent module. Therefore, this optimization scheme can only obtain the local optimal solution for the mixed reconstruction granularity.

On the other hand, according to Corollary~\ref{corollary_lossvar}, the ModCap difference is positively correlated with the difference in reconstruction loss, we can consider the reconstruction loss itself as a capacity metric. This scheme can get the global optimal solution, but it requires a PTQ reconstruction to obtain the reconstruction loss, which is relatively inefficient.

\textbf{Batch size of calibration data.} PTQ requires a small portion of the dataset to do calibration of the quantization parameters and model reconstruction. We note that Eq.~\ref{eq:definition} contains the expectation of the squared Frobenius norm. The expectation represents the average of the random variables, and we take the average of a sampled batch to approximate the expectation in the reconstruction process. The law of large numbers \cite{hsu1947complete} proves that the mean of samples is infinitely close to the expectation when the sample size $N$ tends to infinity, \emph{i.e.},

\begin{equation}
\label{eq:expectation}
\begin{aligned}
&\lim_{N \to +\infty} \frac{1}{N}\sum_{m=1}^N {\Vert f_i^{(n)}(W_i, X_i^{(m)})- f_i^{(n)}(\widetilde{W}_i, \widetilde{X}_i^{(m)})\Vert}_F^2 \\
&=\mathbb{E}\left [ {\Vert f_i^{(n)}(W_i, X_i)- f_i^{(n)}(\widetilde{W}_i, \widetilde{X}_i)\Vert}_F^2 \right ], 
\end{aligned}
\end{equation}

Our experiments show that expanding the batch size of calibration data can improve the accuracy of PTQ. And this trend shows the diminishing marginal utility. Specifically, the improvement of PTQ accuracy slows down as the batch size increases. Details are shown in Sec.~\ref{Ablation_Study}.

\begin{table*}[t!]
\centering
\setlength{\tabcolsep}{1.9mm}{
\begin{tabular}{cccccccc}
\toprule  
Methods & 
W/A &
Res$18$&
Res$50$&
MBV$2$$\times1.0$ &
MBV$2$$\times0.75$ &
MBV$2$$\times0.5$ &
MBV$2$$\times0.35$ \\
\midrule  

Full Prec. & $32$/$32$ & $71.01$  & $76.63$ & $72.20$ & $69.95$ & $64.60$ & $60.08$\\
\midrule  
ACIQ-Mix \cite{banner2019post} & $4$/$4$ & $67.00$  & $73.80$ & - & - & - & -\\
ZeroQ \cite{cai2020zeroq} & $4$/$4$ & $21.71$  & $2.94$ & $26.24$ & - & - & - \\
LAPQ \cite{nahshan2021loss} & $4$/$4$ & $60.30$  & $70.00$ & $49.70$ & - & - & -\\
AdaQuant \cite{hubara2020improving} & $4$/$4$ & $69.60$  & $75.90$ & $47.16$ & - & - & -\\
Bit-Split \cite{wang2020towards} & $4$/$4$ & $ 67.56$  & $73.71$ & - & - & - & -\\
AdaRound \cite{nagel2020up} & $4$/$4$ & $67.96$  & $73.88$ & $61.52$ & $55.32$ & $40.71$ & $35.13$\\
$\text{BRECQ}^{*}$ \cite{li2021brecq} & $4$/$4$ & $68.69$  & $74.88$ & $67.51$ & $62.94$ & $53.02$ & $48.88$\\
\rowcolor{gray!10} Ours+BRECQ & $4$/$4$ & \textbf{69.06 \footnotesize \textcolor{red}{(+0.37)}}  & $74.84$ & \textbf{68.56 \footnotesize \textcolor{red}{(+1.05)}} & \textbf{64.55 \footnotesize \textcolor{red}{(+1.61)}} & \textbf{55.26 \footnotesize \textcolor{red}{(+2.24)}} & \textbf{50.67 \footnotesize \textcolor{red}{(+1.79)}}\\
QDROP \cite{wei2022qdrop} & $4$/$4$ & $69.10$  & $75.03$ & $67.89$ & $63.26$ & $54.19$ & $49.79$\\
\rowcolor{gray!10} Ours+QDROP & $4$/$4$ & \textbf{69.46 \footnotesize \textcolor{red}{(+0.36)}}  & \textbf{75.35 \footnotesize \textcolor{red}{(+0.32)}} & \textbf{68.84 \footnotesize \textcolor{red}{(+0.95)}} & \textbf{64.39 \footnotesize \textcolor{red}{(+1.13)}} & \textbf{55.64 \footnotesize \textcolor{red}{(+1.45)}} & \textbf{50.94 \footnotesize \textcolor{red}{(+1.15)}}\\
\midrule  
LAPQ \cite{nahshan2021loss} & $2$/$4$ & $0.18$  & $0.14$ &$0.13$ & - & - & -\\
AdaQuant \cite{hubara2020improving} & $2$/$4$ & $0.11$  & $0.12$ & $ 0.15$ & - & - & -\\
AdaRound \cite{nagel2020up} & $2$/$4$ & $62.12$  & $66.11$ & $36.31$ & $25.58$ & $15.12$ & $12.46$\\
$\text{BRECQ}^{*}$ \cite{li2021brecq} & $2$/$4$ & $63.71$  & $68.55$ & $52.30$ & $47.14$ & $34.55$ & $30.80$\\
\rowcolor{gray!10} Ours+BRECQ & $2$/$4$ & \textbf{65.61 \footnotesize \textcolor{red}{(+1.9)}}  & \textbf{70.04 \footnotesize \textcolor{red}{(+1.49)}} & \textbf{58.49 \footnotesize \textcolor{red}{(+6.19)}} & \textbf{52.50 \footnotesize \textcolor{red}{(+5.36)}} & \textbf{41.16 \footnotesize \textcolor{red}{(+6.61)}} & \textbf{35.46 \footnotesize \textcolor{red}{(+4.66)}}\\
QDROP \cite{wei2022qdrop} & $2$/$4$ & $64.66$ & $70.08$ & $52.92$ & $49.00$ & $37.13$ & $32.37$\\
\rowcolor{gray!10} Ours+QDROP & $2$/$4$ & \textbf{66.18 \footnotesize \textcolor{red}{(+1.52)}} & \textbf{70.53 \footnotesize \textcolor{red}{(+0.45)}} & \textbf{57.85 \footnotesize \textcolor{red}{(+4.93)}} & \textbf{53.71 \footnotesize \textcolor{red}{(+4.71)}} & \textbf{40.09 \footnotesize \textcolor{red}{(+2.96)}} & \textbf{35.85 \footnotesize \textcolor{red}{(+3.48)}}\\
\midrule  
AdaQuant \cite{hubara2020improving} & $3$/$3$ & $60.09$ & $67.46$ & $2.23$ & - & - & -\\
$\text{AdaRound}^{*}$ \cite{nagel2020up} & $3$/$3$ & $63.91$ & $64.85$ & $34.55$ & $18.16$ & $8.13$ & $4.45$\\
$\text{BRECQ}^{*}$ \cite{li2021brecq} & $3$/$3$ & $64.83$ & $70.06$ & $52.03$ & $45.54$ & $29.79$ & $25.52$\\
\rowcolor{gray!10} Ours+BRECQ & $3$/$3$ & \textbf{65.64 \footnotesize \textcolor{red}{(+0.81)}} & \textbf{70.68 \footnotesize \textcolor{red}{(+0.62)}} & \textbf{57.14 \footnotesize \textcolor{red}{(+5.11)}} & \textbf{50.21 \footnotesize \textcolor{red}{(+4.67)}} & \textbf{35.11 \footnotesize \textcolor{red}{(+5.32)}} & \textbf{30.26 \footnotesize \textcolor{red}{(+4.74)}} \\
QDROP \cite{wei2022qdrop} & $3$/$3$ & $65.56$ & $71.07$ & $54.27$ & $49.26$ & $35.14$ & $29.40$\\
\rowcolor{gray!10} Ours+QDROP & $3$/$3$ & \textbf{66.30 \footnotesize \textcolor{red}{(+0.74)}} & \textbf{71.92 \footnotesize \textcolor{red}{(+0.85)}} & \textbf{58.40 \footnotesize \textcolor{red}{(+4.13)}} & \textbf{51.78 \footnotesize \textcolor{red}{(+2.52)}} & \textbf{38.43 \footnotesize \textcolor{red}{(+3.29)}} & \textbf{32.96 \footnotesize \textcolor{red}{(+3.56)}}\\
\midrule  
$\text{BRECQ}^{*}$ \cite{li2021brecq} & $2$/$2$ & $46.89$ & $40.18$ & $7.03$ & $5.60$ & $1.87$ & $1.62$\\
\rowcolor{gray!10} Ours+BRECQ & $2$/$2$ & \textbf{52.02 \footnotesize \textcolor{red}{(+5.13)}} & \textbf{43.72 \footnotesize \textcolor{red}{(+3.54)}} & \textbf{13.84 \footnotesize \textcolor{red}{(+6.81)}} & \textbf{9.46 \footnotesize \textcolor{red}{(+3.86)}} & \textbf{3.43 \footnotesize \textcolor{red}{(+1.56)}} & \textbf{3.22 \footnotesize \textcolor{red}{(+1.6)}}\\
QDROP \cite{wei2022qdrop} & $2$/$2$ & $51.14$ & $54.74$ & $8.46$ & $8.67$ & $3.31$ & $2.77$\\
\rowcolor{gray!10} Ours+QDROP & $2$/$2$ & \textbf{54.46 \footnotesize \textcolor{red}{(+3.32)}} & \textbf{56.82 \footnotesize \textcolor{red}{(+2.08)}} & \textbf{14.44 \footnotesize \textcolor{red}{(+5.98)}} & \textbf{11.40 \footnotesize \textcolor{red}{(+2.73)}} & \textbf{4.18 \footnotesize \textcolor{red}{(+0.87)}} & \textbf{3.09 \footnotesize \textcolor{red}{(+0.32)}}\\
\bottomrule 
\end{tabular}
}
\caption{A comparison of our algorithm with the State-Of-The-Art method. We combine our algorithm with two latest PTQ algorithms, BRECQ and QDROP, and show significant improvements on a wide range of models. ``*" represents our reproduction of the algorithm based on an open-source codebase in a uniform experimental setup. ``W/A" represents the bit width of the weights and activations, respectively. ``Full Prec." is the full precision pre-trained model. Under different bit configurations, we show in the table the comparison of our algorithm with a wide range of PTQ methods on ResNet-$18$, ResNet-$50$ and different scaled MobileNetV$2$.}
\label{tab:PTQ}
\end{table*}

\section{Experiments}

In this section, we perform a series of experiments to verify the superiority of our algorithm. First, we introduce the implementation details of our experiments. Second, we compare our algorithm quantized to different low bit-widths with the State-Of-The-Art on ImageNet datasets. Finally, we design a variety of ablation experiments to comprehensively analyze the properties of our algorithm, including the Pareto optimum subject to model size, the respective contributions of MRECG and expanded batch size, the phenomenon of diminishing marginal utility of expanded batch size, and loss distribution of different algorithms.

\subsection{Implementation Details}

We validate the performance of our algorithm on the ImageNet dataset \cite{russakovsky2015imagenet}, which consists of $1.2$M training images and $50,000$ validation images. We take $16$ batches of data to perform the reconstruction of PTQ. We fix the batch size to $256$ for ResNet-$18$ and MobileNetV$2$, and $128$ for ResNet-$50$. In addition, our experimental analysis of the expanded batch size is shown in Sec.~\ref{Ablation_Study}. Our data preprocessing follows \cite{he2016deep} as everyone does. We use the Pytorch library \cite{paszke2019pytorch} to complete our algorithm. We perform our experiments on an Nvidia Tesla A100 as well as an Intel(R) Xeon(R) Platinum 8336C CPU.

For PTQ reconstruction, our weight rounding scheme follows \cite{nagel2020up}. The rest of our reconstruction hyper-parameters, such as reconstruction iterations, loss ratios, etc., are consistent with Adaround \cite{nagel2020up}, BRECQ \cite{li2021brecq} and QDrop \cite{wei2022qdrop}. Also, as mentioned in QDrop, BRECQ additionally relaxes the bit-width of the first layer output to $8$ bit, which somehow brings some accuracy gains. We re-executed BRECQ in an experimental setup uniform with all algorithms and obtained the accuracy under different bit configurations of models. Please refer to the supplementary material for other hyper-parameter settings.

\subsection{ImageNet Classification}

We validated our algorithm on a wide range of models, including ResNet-$18$, ResNet-$50$ and MobileNetV$2$ at different scales. As shown in Tab.~\ref{tab:PTQ}, our algorithm outperforms other methods by a large margin on a wide range of models corresponding to different bit configurations. Specifically, we achieve 66.18\% and 57.85\% Top-$1$ accuracy on ResNet-$18$, MobileNetV$2$ with $2$/$4$ bit, which are 1.52\% and 4.93\% higher than SOTA, respectively. Secondly, our method shows strong superiority at lower bits. For example, the accuracy gain of $2$/$4$ bit PTQ quantization on MobileNetV$2$$\times$$0.75$ is significantly higher than that of $4$/$4$ bit (5.36\% vs. 1.61\%). The oscillation problem may be more significant because the low bit causes an increase in the magnitude of PTQ reconstruction error. So smoothing this oscillation at low bit quantization is crucial to the PTQ optimization process. Also, our method shows considerable improvement on MobileNetV$2$ at different scales, implying that the model size does not limit the performance of our algorithm. Besides, we notice that our algorithm is more effective for MobileNetV$2$. Through our observations, the oscillation problem in MobileNetV$2$ is more severe than that of the ResNet family of networks. The deep separable convolution increases the difference in module capacity and thus makes the magnitude of the loss oscillation larger. Therefore, our algorithm achieves a considerable performance improvement by solving the oscillation problem in MobileNetV$2$.

\begin{figure}[t]
\centering
\includegraphics[width=1\linewidth]{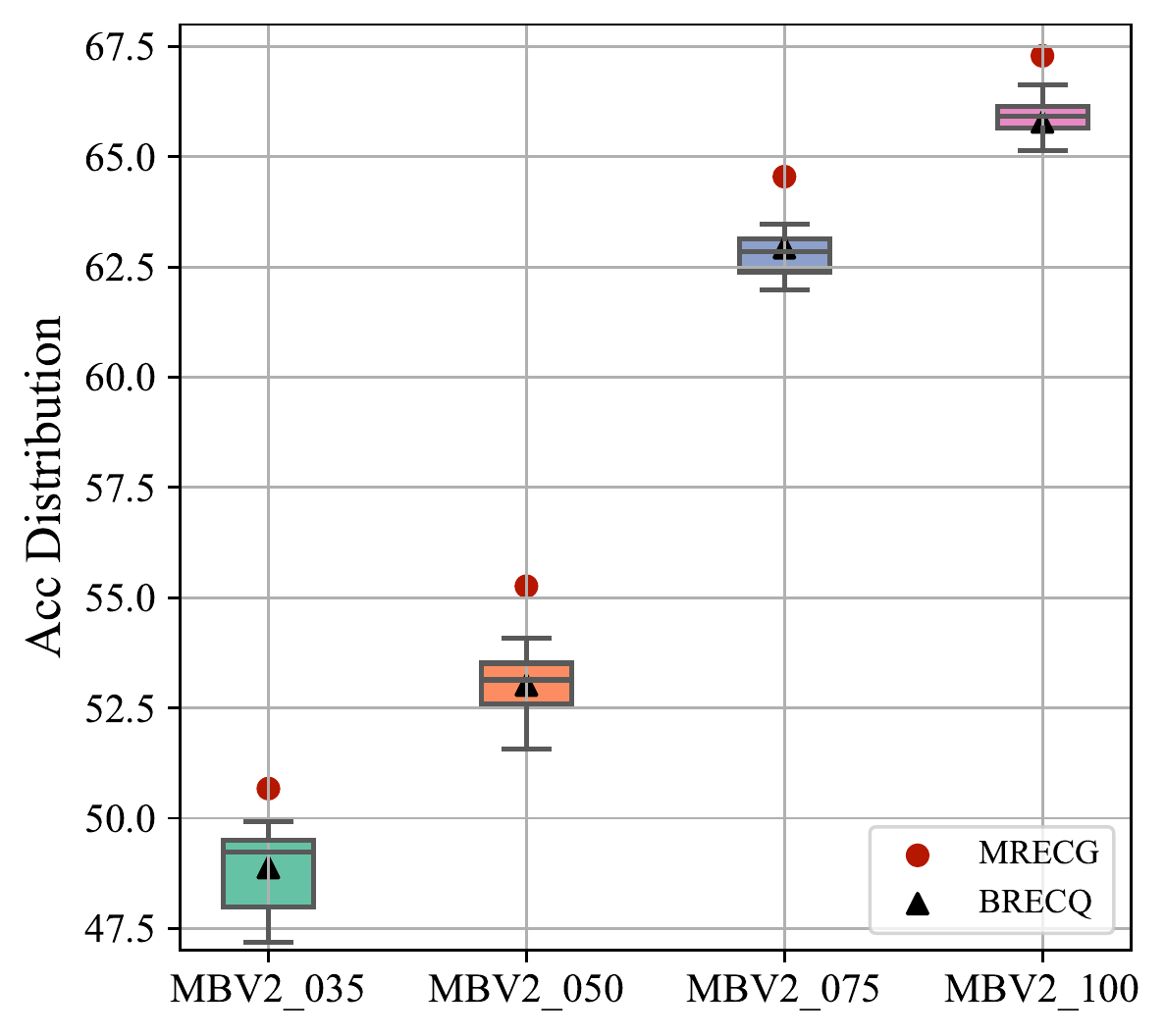}
\caption{Pareto optimality of MRECG. We randomly sampled a large number of mixed reconstruction granularity schemes on four scales of MobileNetV$2$. We combine these schemes with the BRECQ algorithm to recover accuracy. The accuracy distribution of the sampled schemes on different models is given. Also, we mark the accuracy values of MRECG and BRECQ, respectively.}
\label{fig:random_sample}
\end{figure}

\begin{table}[!t]
\centering
\setlength{\tabcolsep}{1.8mm}{
\begin{tabular}{cc}
\toprule  
Methods & 
Top-$1$ Acc(\%)\\
\midrule  
Baseline   & $50.83$ \\
\midrule  
Baseline+MRECG   & $51.91$ \\
Baseline+ExpandBS   & $51.57$ \\
\midrule  
Baseline+both(Ours) & $52.25$ \\
\bottomrule 
\end{tabular}}
\caption{Top-1 accuracy (\%) for different component combinations. Baseline is $0.5$ scaled MobileNetV$2$ quantized to $4/4$ bit using the BRECQ algorithm. MRECG is mixed reconstruction granularity. ExpandBS represents expanded batch size.}
\label{tab:component_combinations}
\end{table}

\begin{figure}[t]
\centering
\includegraphics[width=1\linewidth]{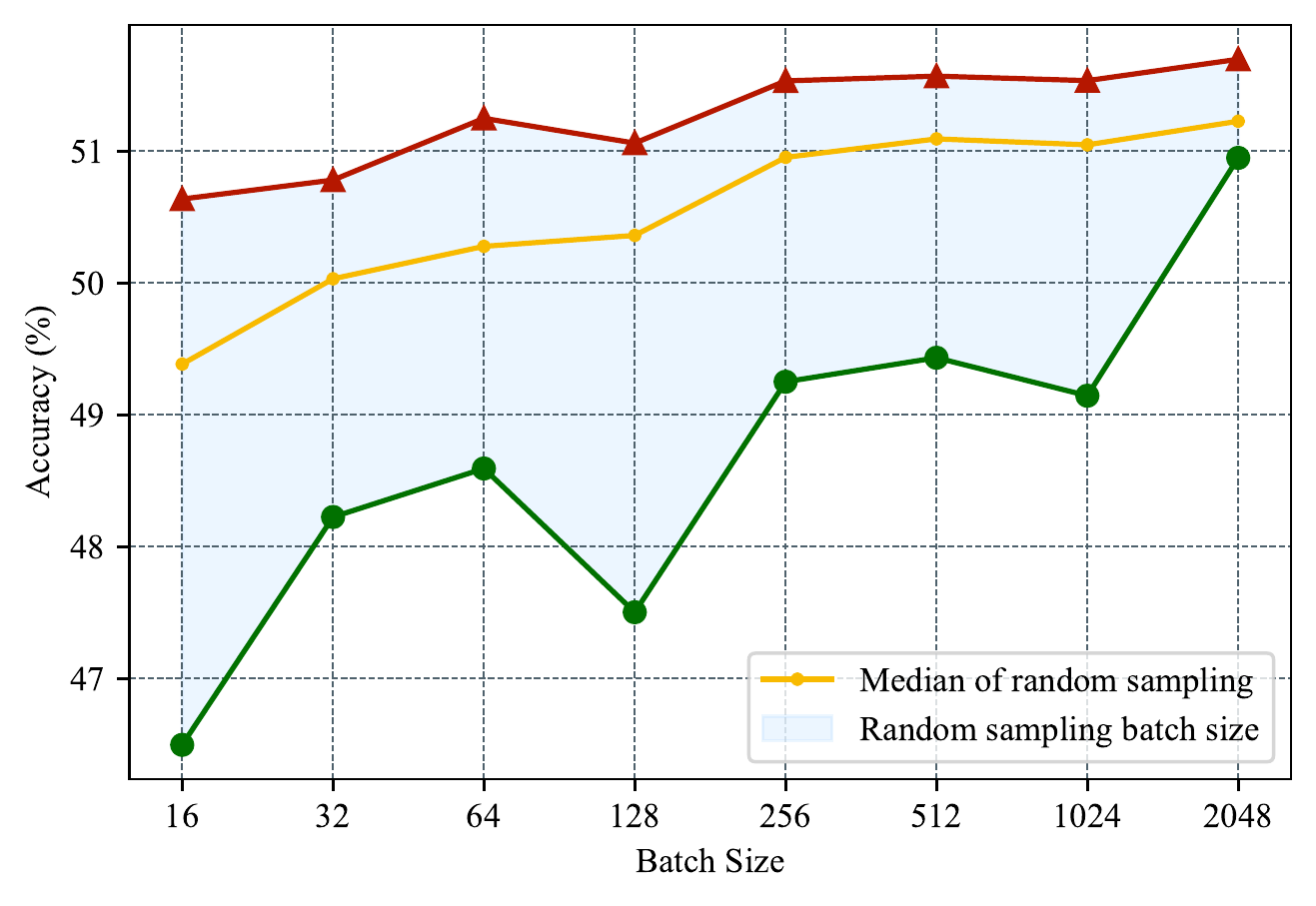}
\caption{Impact of expanding batch size on PTQ accuracy. We sample some batch sizes on $4$/$4$ bit MobileNetV$2$ and recover the accuracy. We demonstrate the median of sampled accuracy as the batch size increases to prevent the effect of outlier points.}
\label{fig:random_sample_bs}
\end{figure}

\subsection{Ablation Study}
\label{Ablation_Study}

\noindent
\textbf{Pareto optimal.} We randomly sample a large number of mixed reconstruction granularity schemes on different scaled MobileNetV$2$ and optimize these schemes with $4$/$4$ bit PTQ using the BRECQ algorithm to obtain the sampling accuracy. We plot the accuracy distribution of these schemes in Fig.~\ref{fig:random_sample}. Also we mark the accuracies of MRECG and BRECQ algorithms on MobileNetV$2$ at different scales, respectively. From Fig.~\ref{fig:random_sample}, we can observe that MRECG has stable accuracy improvement on MobileNetV$2$ at all scales. Overall, MRECG achieves the Pareto optimal state under the limitation of model size.

\noindent
\textbf{Component Contribution.} Here we examine how much mixed reconstruction granularity and expanded batch size contribute to accuracy, respectively. We take the $0.5$ scaled MobileNetV$2$ quantized to $4/4$ bit by BRECQ as the baseline. As shown in Tab.~\ref{tab:component_combinations}, the mixed reconstruction granularity and expanded batch size have $1.08\%$ and $0.74\%$ accuracy improvement, respectively. Meanwhile, the combination of the two methods further boosts model performance.

\noindent
\textbf{Diminishing marginal utility of expanding batch size.} By the law of large numbers \cite{hsu1947complete}, when the number of samples increases, the mean of samples converges to the expectation in Eq.~\ref{eq:definition} thus yielding a smaller approximation error. However, when the sample size is large enough, the accuracy gain from the reduction of the approximation error is negligible. In other words, expanding batch size presents a trend of diminishing marginal utility. In Fig.~\ref{fig:random_sample_bs}, we randomly sample some batch sizes and obtained PTQ accuracy of $4/4$ bit by BRECQ on $0.5$ scaled MobileNetV$2$. We demonstrate the median of the sampling accuracy to prevent the effect of outliers. We notice that the sampling accuracy fluctuates more when the batch size is small, implying that the approximation error of a smaller batch size generates larger noise. This situation is moderated as the batch size is expanded. In addition, we observe an increase and then stabilization of the median accuracy, implying that expanding the batch size can bring accuracy gains to PTQ. Meanwhile, this gain is constrained by the diminishing marginal utility.

\noindent
\textbf{Loss distribution.} As in Fig.~\ref{fig:all_loss_dis}, we present the loss distribution of different algorithms. From the distribution of Adaround we can see that it has the largest oscillation amplitude. Therefore, in the deeper layers of the model, the reconstruction loss of Adaround increases rapidly. In addition, BRECQ uses block reconstruction for PTQ optimization, which somehow alleviates the loss oscillation problem and therefore brings performance improvement. Our mixed reconstruction granularity smoothes out the loss oscillations by jointly optimizing adjacent modules with large differences in capacity. It can be seen that the loss variation of MRECG is more stable compared to Adaround and BRECQ. Moreover, when comparing the MRECG loss in different scenarios, we find that the MRECG obtained based on the data-free scenario still has a small amplitude loss oscillation. We believe that this modest oscillation is related to the fact that the ModCap MRECG can only jointly optimize two adjacent modules, which can only achieve local optimality. In contrast, the loss MRECG has the smoothest loss curve. However, this global optimum comes at the cost of prolonging the time to obtain the quantized model.

\begin{figure}[t]
\centering
\includegraphics[width=1\linewidth]{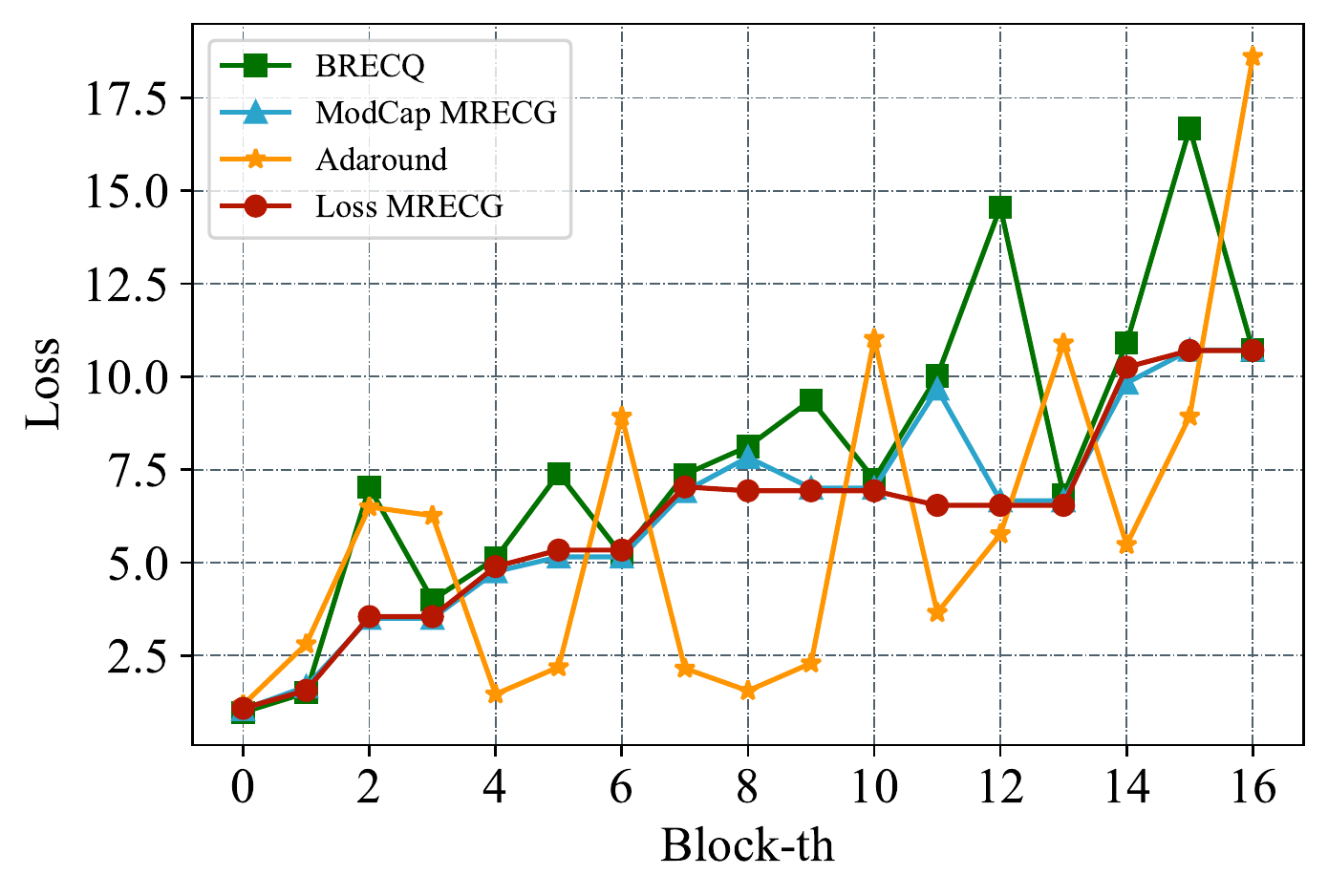}
\caption{The reconstruction loss distributions of different algorithms for $4$/$4$ bit configurations on $0.5$ scaled MobileNetV$2$, including Adaround, BRECQ, and MRECG constructed using ModCap and Loss metric for data-free or data-dependency scenarios. we omit the final fully connected layer loss for clarity.}
\label{fig:all_loss_dis}
\end{figure}

\section{Conclusion}

In this paper, we discover an oscillation problem in the PTQ optimization process for the first time, which has been neglected in all previous PTQ algorithms. We then theoretically analyze that this oscillation is caused by the difference in the capacity of adjacent modules. Meanwhile, we observe that the final reconstruction error is positively correlated with the largest value of reconstruction error. Therefore, we construct a mixed reconstruction granularity optimization problem to smoothen the loss oscillations by jointly optimizing the adjacent modules with large capacity differences in data-dependent and data-free scenarios, which reduces the final reconstruction error. In addition, we observe that increasing the batch size of calibration data can reduce the expectation approximation error of the objective function. And this gain is of diminishing marginal utility. We validate the effectiveness of our method on extensive models with different bit configurations. 


\section*{Acknowledgments}
This work was supported by National Key R\&D Program of China (No.2022ZD0118202), the National Science Fund for Distinguished Young Scholars (No.62025603), the National Natural Science Foundation of China (No. U21B2037, No. U22B2051, No. 62176222, No. 62176223, No. 62176226, No. 62072386, No. 62072387, No. 62072389, No. 62002305 and No. 62272401), and the Natural Science Foundation of Fujian Province of China (No.2021J01002,  No.2022J06001), Guangdong Basic and Applied Basic Research Foundation (No. 2019B1515120049), CAAI-Huawei MindSpore Open Fund.


{\small
\bibliographystyle{ieee_fullname}
\bibliography{egbib}
}

\clearpage
\section*{Supplementary Material}

\appendix

\section{Proof of Theorem 1}
\label{proof:Theorem_1}

\begin{Theorem}
\label{the:error_cumul} 
Given a pre-trained model and input data. If two adjacent modules are equivalent, we have,
\begin{equation}\label{eq:inequality}
	 \mathcal{L}(W_i, X_i) \leq \mathcal{L}(W_{i+1}, X_{i+1}). 
\end{equation}

\end{Theorem}

\begin{proof}
When two adjacent equivalent modules contain only one convolutional layer, that is, 
\begin{equation}
\begin{split}
\label{eq:one_conv}
&f_{i+1}^{(n_{i+1})}(W_i,X_i) \\
=& f_{i}^{(n_i)}(W_i,X_i) \\
=& f_i^{(1)}(W_i, X_i) \\
=& W_iX_i.
\end{split}
\end{equation}
Where we transform the tensor convolution into a matrix multiplication for simplicity. This transform is common in the practical implementation of convolution and is usually accompanied by the General Matrix Multiplication (GEMM) for the practical speedup \cite{ahmad2020optimizing}. Therefore, 

\begin{equation}
\begin{split}
\label{eq:loss_i+1}
&\mathcal{L}(W_{i+1}, X_{i+1}) \\
= &\mathbb{E}\left [ {\Vert W_{i+1}X_{i+1}- \widetilde{W}_{i+1}\widetilde{X}_{i+1}\Vert}_F^2 \right ] \\
=& \mathbb{E}\left [ {\Vert W_{i+1}X_{i+1}- (W_{i+1}+\Delta W_{i+1})(X_{i+1}+\Delta X_{i+1})\Vert}_F^2 \right ] \\
=& \mathbb{E}\left [ {\Vert W_{i+1}\Delta X_{i+1}+ \Delta W_{i+1}X_{i+1}+\Delta W_{i+1}\Delta X_{i+1}\Vert}_F^2 \right ] \\
\approx & \mathbb{E}\left [ {\Vert W_{i+1}\Delta X_{i+1}+ \Delta W_{i+1}X_{i+1}\Vert}_F^2 \right ], 
\end{split}
\end{equation}
where $\Delta W$ and $\Delta X$ are the quantization errors of $W$ and $X$. We ignore the higher order term $\Delta W\Delta X$ of the quantization error. Then, 

\begin{equation}\label{eq:loss_i+1_continue}
\begin{aligned}
\begin{split}
  &\mathbb{E}\left [ {\Vert W_{i+1}\Delta X_{i+1}+ \Delta W_{i+1}X_{i+1}\Vert}_F^2 \right ] \\
  =& \mathbb{E} \left [ 
\sum_m \sum_n {\left ( w_{i+1}^{(m)}\Delta x_{i+1}^{(n)}+\Delta w_{i+1}^{(m)}x_{i+1}^{(n)} \right )}^2 \right ] \\
=& \mathbb{E} \left [ 
\sum_m \sum_n {\left ( w_{i+1}^{(m)}\Delta x_{i+1}^{(n)} \right )}^2  \right ] + \mathbb{E} \left [ 
\sum_m \sum_n {\left ( \Delta w_{i+1}^{(m)}x_{i+1}^{(n)} \right )}^2  \right ] \\
&+ 2\mathbb{E} \left [ 
\sum_m \sum_n  w_{i+1}^{(m)}\Delta x_{i+1}^{(n)}\Delta w_{i+1}^{(m)}x_{i+1}^{(n)}   \right ] \\
=&  
\sum_m \sum_n \mathbb{E} \left [ {\left ( w_{i+1}^{(m)}\Delta x_{i+1}^{(n)} \right )}^2  \right ] + 
\sum_m \sum_n \mathbb{E} \left [ {\left ( \Delta w_{i+1}^{(m)}x_{i+1}^{(n)} \right )}^2  \right ] \\
&+ 2\sum_m \sum_n  \mathbb{E} \left [ w_{i+1}^{(m)}\Delta x_{i+1}^{(n)}\Delta w_{i+1}^{(m)}x_{i+1}^{(n)}   \right ], \\
\end{split}
\end{aligned}
\end{equation}
where $w_{i+1}^{(m)}$ is the $m$-th row vector of $W$, $x_{i+1}^{(n)}$ is the $n$-th column vector of $X$, and others as well. Since the pre-trained model and the input are given, $w_{i+1}^{(m)}$ and $x_{i+1}^{(n)}$ are constant vectors in the PTQ optimization process, so that,

\begin{equation}\label{eq:equivalent}
\begin{aligned}
\begin{split}
\mathbb{E} \left [ {\left ( w_{i+1}^{(m)}\Delta x_{i+1}^{(n)} \right )}^2  \right ] &= {\left ( w_{i+1}^{(m)}\mathbb{E} \left [ \Delta x_{i+1}^{(n)} \right ] \right )}^2, \\
\mathbb{E} \left [ {\left ( \Delta w_{i+1}^{(m)}x_{i+1}^{(n)} \right )}^2  \right ] &=  {\left ( \mathbb{E} \left [\Delta w_{i+1}^{(m)}\right ]x_{i+1}^{(n)} \right )}^2  , \\
\mathbb{E} \left [ w_{i+1}^{(m)}\Delta x_{i+1}^{(n)}\Delta w_{i+1}^{(m)}x_{i+1}^{(n)}   \right ] &= w_{i+1}^{(m)}\mathbb{E} \left [\Delta x_{i+1}^{(n)}\Delta w_{i+1}^{(m)}\right ]x_{i+1}^{(n)} . \\
\end{split}
\end{aligned}
\end{equation}
Since the two adjacent modules are equivalent and each module contains a batchnorm layer \cite{ioffe2015batch}, we thus consider the full precision weights and activations of the two modules to be identically distributed, i.e.,

\begin{equation}\label{eq:identically_distributed}
\begin{aligned}
\begin{split}
{\mathbb{E}}_{W_i\sim \mathcal{D}_w} \left [ w_i^{(m)} \right ] &= {\mathbb{E}}_{W_{i+1}\sim \mathcal{D}_w} \left [ w_{i+1}^{(m)} \right ], \\
{\mathbb{E}}_{X_i\sim \mathcal{D}_x} \left [ x_i^{(n)} \right ] &= {\mathbb{E}}_{X_{i+1}\sim \mathcal{D}_x} \left [ x_{i+1}^{(n)} \right ]. \\
\end{split}
\end{aligned}
\end{equation}
Meanwhile, due to the accumulative effect of quantification errors, we have, 

\begin{equation}\label{eq:QE_accumulative_effect}
\begin{aligned}
\begin{split}
\mathbb{E} \left [ \Delta x_{i+1}^{(n)} \right ]>\mathbb{E} \left [ \Delta x_{i}^{(n)} \right ], \mathbb{E} \left [\Delta w_{i+1}^{(m)}\right ]>\mathbb{E} \left [\Delta w_{i}^{(m)}\right ].
\end{split}
\end{aligned}
\end{equation}
Therefore, Eq.~\ref{eq:inequality} holds when the two modules contain one convolutional layer. Subsequently, suppose that Eq.~\ref{eq:inequality} holds when the module contains $n$ convolutional layers, we have, 
\begin{equation}\label{eq:n_conv_loss}
\begin{aligned}
\begin{split}
&\mathbb{E}\left [ { \left \Vert f_{i+1}^{(n)}(W_{i+1}, X_{i+1})- f_{i+1}^{(n)}(\widetilde{W}_{i+1}, \widetilde{X}_{i+1}) \right \Vert}_F^2 \right ] \\
&> \mathbb{E}\left [ { \left \Vert f_{i}^{(n)}(W_{i}, X_{i})- f_{i}^{(n)}(\widetilde{W}_{i}, \widetilde{X}_{i}) \right \Vert}_F^2 \right ]. \\
\end{split}
\end{aligned}
\end{equation}
We set, 
\begin{equation}\label{eq:Equivalent_representation}
\begin{aligned}
\begin{split}
f_{i+1}^{(n)}(W_{i+1},X_{i+1})=X_{i+1}^{(n+1)}, \\
f_{i+1}^{(n)}(\widetilde{W}_{i+1},\widetilde{X}_{i+1})=\widetilde{X}_{i+1}^{(n+1)}. \\
\end{split}
\end{aligned}
\end{equation}
Our supposition is equivalently converted to,
\begin{equation}\label{eq:n_conv_loss}
\begin{aligned}
\begin{split}
\mathbb{E}\left [ { \left \Vert X_{i+1}^{(n+1)}- \widetilde{X}_{i+1}^{(n+1)} \right \Vert}_F^2 \right ] &> \mathbb{E}\left [ { \left \Vert X_{i}^{(n+1)}- \widetilde{X}_{i}^{(n+1)} \right \Vert}_F^2 \right ] \\
\mathbb{E}\left [ { \left \Vert \Delta X_{i+1}^{(n+1)} \right \Vert}_F^2 \right ] &> \mathbb{E}\left [ { \left \Vert \Delta X_{i}^{(n+1)} \right \Vert}_F^2 \right ]. \\
\end{split}
\end{aligned}
\end{equation}
When the module contains $n+1$ convolutional layers, 
\begin{equation}\label{eq:n+1_conv_loss}
\small
\begin{aligned}
\begin{split}
&\mathbb{E}\left [ { \left \Vert f_{i+1}^{(n+1)}(W_{i+1}, X_{i+1})- f_{i+1}^{(n+1)}(\widetilde{W}_{i+1}, \widetilde{X}_{i+1}) \right \Vert}_F^2 \right ] \\
=& \mathbb{E}\left [ { \left \Vert W_{i+1}^{(n+1)}f_{i+1}^{(n)}(W_{i+1}, X_{i+1})- \widetilde{W}_{i+1}^{(n+1)}f_{i+1}^{(n)}(\widetilde{W}_{i+1}, \widetilde{X}_{i+1}) \right \Vert}_F^2 \right ] \\
=& \mathbb{E}\left [ { \left \Vert W_{i+1}^{(n+1)}X_{i+1}^{(n+1)}- \widetilde{W}_{i+1}^{(n+1)}\widetilde{X}_{i+1}^{(n+1)} \right \Vert}_F^2 \right ]. \\
\end{split}
\end{aligned}
\end{equation}
Similar to the derivation of Eqs.~\ref{eq:loss_i+1_continue}-\ref{eq:QE_accumulative_effect}, we can obtain that Eq.~\ref{eq:inequality} also holds when the module contains $n+1$ convolutional layers. Therefore, by inductive reasoning, the Theorem~\ref{the:error_cumul} is proved.

\end{proof}

\section{Proof of Corollary 1}
\label{proof:Corollary_1}

\begin{Corollary}
\label{corollary_lossvar}
Suppose two adjacent modules be topologically homogeneous. If the module capacity of the later module is large enough, the loss will decrease. Conversely, if the latter module capacity is smaller than the former, then the accumulation effect of quantization error is exacerbated.
\end{Corollary}

\begin{proof}
We consider an extreme case, where the module is equivalent to no quantization if the bit-width $b_i$ in the ModCap is taken to be $32$ bits and the activation value is not quantized. In this case, the quantization error of this module is $0$. Therefore, when the module capacity is large enough, the quantization error of the module will converge to $0$, i.e.,

\begin{equation}\label{eq:limit_qe}
\small
\begin{aligned}
\begin{split}
\lim_{ModCap \to m} \Delta W = \lim_{ModCap \to m} \Delta X = 0,
\end{split}
\end{aligned}
\end{equation}
where $m$ is the module capacity of the full precision module. $\Delta W, \Delta X$ is the quantization error of the module weights and activation. Therefore, for arbitrary $\epsilon > 0$, there exists a module capacity $ n $ such that $\mathcal{L}(W_{i+1},X_{i+1}) < \epsilon$ when $ModCap>n$. Since $\epsilon$ is arbitrary, we take $\epsilon ={\epsilon}_1< \mathcal{L}(W_i,X_i)$. Consequently, we have 
\begin{equation}\label{ineq:limit_qe}
\small
\begin{aligned}
\begin{split}
\mathcal{L}(W_{i+1},X_{i+1}) < {\epsilon}_1 < \mathcal{L}(W_i,X_i).
\end{split}
\end{aligned}
\end{equation}
Conversely, by Theorem~\ref{the:error_cumul}, when the capacity of the later module is the same as the earlier one, due to the accumulative effect of the quantization error, we have

\begin{equation}
	 \mathcal{L}(W_i, X_i) \leq \mathcal{L}(W_{i+1}, X_{i+1}). 
\end{equation}
If the capacity of the later module is smaller than the earlier one, the quantization error of the module will increase, i.e., 
\begin{equation}
\begin{aligned}
\begin{split}
\lim_{ModCap \to 0} \Delta W = \lim_{ModCap \to 0} \Delta X = \zeta,
\end{split}
\end{aligned}
\end{equation}
where $\zeta$ is the upper bound on the quantization error of the module.
Therefore, for arbitrary $\epsilon > 0$, there exists a module capacity $ n $ such that ${\mathcal{L}}^{'}(W_{i+1},X_{i+1}) > \zeta-\epsilon$ when $ModCap<n$. Since $\epsilon$ is arbitrary, we take $\epsilon = {\epsilon}_1 = \zeta-\mathcal{L}(W_{i+1},X_{i+1})$. Consequently, we have
\begin{equation}\label{ineq:limit_qe_aggravate}
\small
\begin{aligned}
\begin{split}
{\mathcal{L}}^{'}(W_{i+1},X_{i+1}) > \zeta-{\epsilon}_1 = \mathcal{L}(W_{i+1},X_{i+1})>\mathcal{L}(W_{i},X_{i}).
\end{split}
\end{aligned}
\end{equation}
Therefore, the corollary is proved. 

\end{proof}

\begin{figure}[t]
\centering
\includegraphics[width=1.0\linewidth]{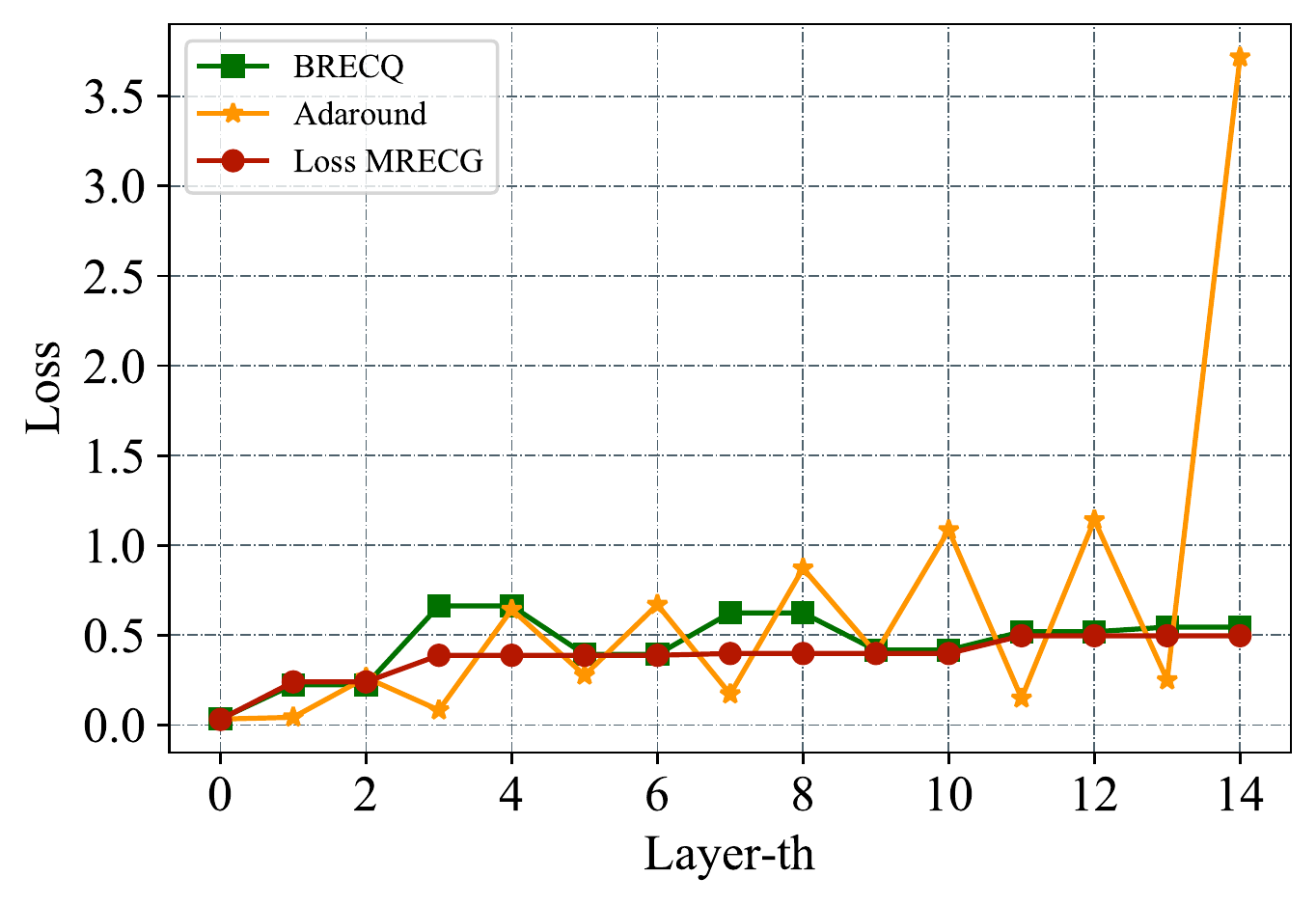}
\caption{Loss distribution of different algorithms on ResNet-$18$. We ignore the reconstruction loss of the final block and the fully connected layer. Convolutional layers within a block of the block reconstruction algorithm share the reconstruction loss of that block.}
\label{fig:loss_distribute_r18}
\end{figure}

\section{Loss distribution of ResNet-18}
\label{loss_distri_r18}

In Fig.~\ref{fig:loss_distribute_r18} we present the reconstruction loss distribution of ResNet-$18$ quantized to $4$/$4$ bit. AdaRound shows the most drastic loss oscillations. As a result, this causes a sharp increase in reconstruction error due to irretrievable information loss. The block reconstruction strategy of BRECQ mitigates the oscillations in AdaRound, but still shows small oscillations within some layers. For example, oscillations occur between the $4$-th layer and $5$-th layer. MRECG completely smoothes out the oscillations, allowing the reconstruction loss of ResNet-$18$ to be smoothly incremented by the accumulation of quantization errors.

\begin{table}[!t]
\centering

\setlength{\tabcolsep}{1.8mm}{
\begin{subtable}[ResNet-$18$]{1\linewidth}
\caption{ResNet-$18$}
\centering
\begin{tabular}{ccc}
\toprule  
Methods &
Searching Time &
Top-$1$ Acc(\%)\\
\midrule  
BRECQ   & $\textbf{0}$ & $64.83$\\
QDROP   & $\textbf{0}$ & $65.56$\\
ModCap MRECG   & $\textbf{0}$ & $66.07$\\
Loss MRECG   & $19.1~\text{mins}$ & $\textbf{66.30}$\\
\bottomrule 
\end{tabular}
\end{subtable}

\begin{subtable}[ResNet-$50$]{1\linewidth}
\caption{ResNet-$50$}
\centering
\begin{tabular}{ccc}
\toprule  
Methods &
Searching Time &
Top-$1$ Acc(\%)\\
\midrule  
BRECQ   & $\textbf{0}$ & $70.06$\\
QDROP   & $\textbf{0}$ & $71.07$\\
ModCap MRECG   & $\textbf{0}$ & $71.65$\\
Loss MRECG   & $60.9~\text{mins}$ & $\textbf{71.92}$\\
\bottomrule 
\end{tabular}
\end{subtable}

\begin{subtable}[MobileNetV$2$]{1\linewidth}
\caption{MobileNetV$2$$\times0.5$}
\centering
\begin{tabular}{ccc}
\toprule  
Methods &
Searching Time &
Top-$1$ Acc(\%)\\
\midrule  
BRECQ   & $\textbf{0}$ & $29.79$\\
QDROP   & $\textbf{0}$ & $35.14$\\
ModCap MRECG   & $\textbf{0}$ & $38.02$\\
Loss MRECG   & $39.5~\text{mins}$ & $\textbf{38.43}$\\
\bottomrule 
\end{tabular}
\end{subtable}
}
\caption{Searching time of reconstruction granularity and Top-$1$ accuracy for different algorithms on ResNet-$18$, ResNet-$50$ and MobileNetV$2$. All weights and activations are quantized to $3$ bit.}
\label{tab:search_time}
\end{table}

\begin{table}[t!]
\centering
\setlength{\tabcolsep}{3.7mm}{
\begin{tabular}{ccc}
\toprule  
Methods & 
W/A &
Reg$600$M \\
\midrule  

Full Prec. & $32$/$32$ & $73.52$ \\
\midrule  
ZeroQ \cite{cai2020zeroq} & $4$/$4$ & $28.54$ \\
LAPQ \cite{nahshan2021loss} & $4$/$4$ & $57.71$ \\
AdaRound \cite{nagel2020up} & $4$/$4$ & $68.20$ \\
$\text{BRECQ}^{*}$ \cite{li2021brecq} & $4$/$4$ & $70.44$ \\
QDROP \cite{wei2022qdrop} & $4$/$4$ & $70.62$ \\
\rowcolor{gray!10} Ours+QDROP & $4$/$4$ & \textbf{71.22 \footnotesize \textcolor{red}{(+0.60)}} \\
\midrule  
LAPQ \cite{nahshan2021loss} & $2$/$4$ & $0.17$ \\
AdaRound \cite{nagel2020up} & $2$/$4$ & $57.00$ \\
$\text{BRECQ}^{*}$ \cite{li2021brecq} & $2$/$4$ & $61.77$ \\
QDROP \cite{wei2022qdrop} & $2$/$4$ & $63.10$ \\
\rowcolor{gray!10} Ours+QDROP & $2$/$4$ & \textbf{65.16 \footnotesize \textcolor{red}{(+2.06)}} \\
\midrule  
$\text{AdaRound}^{*}$ \cite{nagel2020up} & $3$/$3$ & $58.29$ \\
$\text{BRECQ}^{*}$ \cite{li2021brecq} & $3$/$3$ & $62.61$ \\
QDROP \cite{wei2022qdrop} & $3$/$3$ & $64.53$ \\
\rowcolor{gray!10} Ours+QDROP & $3$/$3$ & \textbf{66.08 \footnotesize \textcolor{red}{(+1.55)}} \\
\midrule  
$\text{BRECQ}^{*}$ \cite{li2021brecq} & $2$/$2$ & $28.89$ \\
QDROP \cite{wei2022qdrop} & $2$/$2$ & $38.90$ \\
\rowcolor{gray!10} Ours+QDROP & $2$/$2$ & \textbf{43.67 \footnotesize \textcolor{red}{(+4.77)}} \\
\bottomrule 
\end{tabular}
}
\caption{A comparison of RegNet with the State-Of-The-Art method. "*" indicates that we reproduce the algorithm in a uniform experimental setup based on open-source code. W/A represents the weights and activations bit width, respectively. Under different bit configurations, we show the comparison of our algorithm with a wide range of PTQ methods on RegNet.}
\label{tab:regnet}
\end{table}

\section{MRECG searching time}
\label{search_time}

The mixed reconstruction granularity based on the loss metric requires a small portion of the data for model reconstruction to obtain the loss distribution. As shown in Tab.~\ref{tab:search_time}, the searching time of Loss MRECG on ResNet-$18$, ResNet-$50$ and MobileNetV$2$ are $19.1$ mins, $60.9$ mins and $39.5$ mins, respectively. Loss MRECG achieves the global optimum with a small time overhead. ModCap MRECG does not require the PTQ reconstruction process and thus it is more efficient. Meanwhile, the locally optimal ModCap MRECG also outperforms the previous SOTA method and has a small performance degradation compared to the global optimum.

\section{Classification accuracy of RegNet}
\label{RegNet}

We complement the performance of MRECG on RegNet \cite{radosavovic2020designing}. We quantize the Reg600M model to different bit configurations. As shown in Tab.~\ref{tab:regnet}, MRECG achieves optimality for different bit configurations. Specifically, we obtain $65.16\%$ Top-$1$ accuracy on Reg600M with $2$/$4$ bit, which is $2.06\%$ higher than QDROP.

\section{Implementation details}
\label{Implementation——details}

For the hyper-parameters of the reconstruction, we keep the same as in the previous approaches \cite{li2021brecq, wei2022qdrop}. Specifically, the number of reconstruction iterations in each module is $20,000$, and we set a consistent linearly decreasing temperature b, which ranges from $20$ to $2$. We apply a loss ratio of $0.01$ and $0.1$ in ResNet and MobileNetV$2$, respectively, to balance the reconstruction loss and rounding loss. In the combination of MRECG and QDROP, we adopt $0.5$ probability for each element to decide whether to quantize or retain full precision as described in QDROP. Our pre-trained models for ResNet-$18$, ResNet-$50$ and MobileNetV$2$ are from the PyTorch library \cite{paszke2019pytorch}. Our different scaled MobileNetV2 are obtained by our own training. The number of joint optimization modules $k$ is set to $2$,$4$,$7$ in ResNet-$18$, ResNet-$50$ and MobileNetV$2$, respectively.

\end{document}